\documentclass{article}

\usepackage[nonatbib, preprint]{nips_2018}

\usepackage{tcolorbox}
\usepackage{times}
\usepackage{cite}
\usepackage{subfigure}
\usepackage{graphicx}
\usepackage{amsmath}
\usepackage{bm}
\usepackage{url}
\usepackage{stmaryrd}
\usepackage{balance}
\usepackage{amssymb}
\usepackage{pgfplots}
\usepackage{pifont}
\usepackage[usenames,dvipsnames]{pstricks}
\usepackage{epsfig}
\usepackage{booktabs}
\usepackage{pgffor}
\usepackage{tikz}
\usepackage{multirow}
\usepackage{bbm}

\newcommand{\mb}{\mathbf}

\newtheorem{theo}{\textsc{Theorem}}

\newtheorem{proof}{\textsc{Proof}}
\newtheorem{defn}{\textsc{Definition}}

\newcommand{\our}{\textsc{LNN}}

\newcommand{\linemodel}{\textsc{LINE}}
\newcommand{\deepwalk}{\textsc{DeepWalk}}

\newcommand{\hpe}{\textsc{HPE}}
\newcommand{\walklets}{\textsc{WalkLets}}
\newcommand{\app}{\textsc{App}}
\newcommand{\mf}{\textsc{MF}}
\newcommand{\bpr}{\textsc{Bpr}}

\usepackage{algorithm}
\usepackage{algorithmic}

\title{Deep Loopy Neural Network Model for Graph Structured Data Representation Learning}

\author{Jiawei~Zhang\\
$^\star$IFM Lab, Florida State University, FL, USA\\
jzhang@cs.fsu.edu}

\begin{document}

\maketitle


\begin{abstract}

Existing deep learning models may encounter great challenges in handling graph structured data. In this paper, we introduce a new deep learning model for graph data specifically, namely the \textit{deep loopy neural network}. Significantly different from the previous deep models, inside the \textit{deep loopy neural network}, there exist a large number of loops created by the extensive connections among nodes in the input graph data, which makes model learning an infeasible task. To resolve such a problem, in this paper, we will introduce a new learning algorithm for the \textit{deep loopy neural network} specifically. Instead of learning the model variables based on the original model, in the proposed learning algorithm, errors will be back-propagated through the edges in a group of extracted spanning trees. Extensive numerical experiments have been done on several real-world graph datasets, and the experimental results demonstrate the effectiveness of both the proposed model and the learning algorithm in handling graph data.

\end{abstract}

\section{Introduction}\label{sec:intro}

Formally, a loopy neural network denotes a neural network model involving loops among neurons in its architecture. \textit{Deep loopy neural network} is a novel learning model proposed for graph structured data specifically in this paper. Given a graph data $G = (\mathcal{V}, \mathcal{E})$ (where $\mathcal{V}$ and $\mathcal{E}$ denote the node and edge sets respectively), the architecture of \textit{deep neural nets} constructed for $G$ can be described as follows:

\begin{defn}
(\textbf{Deep Loopy Neural Network}): Formally, we can represent a \textit{deep loopy neural network} model constructed for network $G = (\mathcal{V}, \mathcal{E})$ as $\mathcal{G} = (\mathcal{N}, \mathcal{L})$, where $\mathcal{N}$ covers the set of neurons and $\mathcal{L}$ includes the connections among the neurons defined based on graph $G$. 
\end{defn}

More specifically, the neurons covered in set $\mathcal{N}$ can be categorized into several layers: (1) input layer $\mathcal{X}$, (2) hidden layers $\mathcal{H} = \bigcup_{l=1}^k \mathcal{H}^{(l)}$, and (3) output layer $\mathcal{Y}$, where $\mathcal{X} = \{\mb{x}_i\}_{v_i \in \mathcal{V}}$, $\mathcal{H}^{(l)} = \{\mb{h}_i^{(l)}\}_{v_i \in \mathcal{V}}, \forall l \in \{1, 2, \cdots, k\}$, $\mathcal{Y} = \{\mb{y}_i\}_{v_i \in \mathcal{V}}$ and $k$ denotes the hidden layer depth. Vector $\mb{x}_{i} \in \mathbb{R}^n$, $\mb{h}_i^{(l)} \in \mathbb{R}^{m^{(l)}}$ and $\mb{y}_{i} \in \mathbb{R}^d$ denote the feature, hidden state (at layer $l$) and label vectors of node $v_i \in \mathcal{V}$ in the input graph respectively. Meanwhile, the connections among the neurons in $\mathcal{L}$ can be divided into (1) intra-layer neuron connections $\mathcal{L}_{a}$ (which connect neurons within the same layers), and (2) inter-layer neuron connections $\mathcal{L}_{e}$ (which connect neurons across layers), where $\mathcal{L}_{a} = \bigcup_{l=1}^k \{(\mb{h}^{(l)}_i, \mb{h}^{(l)}_j)\}_{(v_i, v_j) \in \mathcal{E}}$ and $\mathcal{L}_{e} = \mathcal{L}_{x,h_1}  \cup \mathcal{L}_{h_1,h_2} \cup \cdots \cup \mathcal{L}_{h_k,y}$. In the notation, $\mathcal{L}_{x,h_1}$ covers the connections between neurons in the input layer and the first hidden layer, and so forth for the other neuron connection sets.

In Figure~\ref{fig:framework}, we show an example of a \textit{deep loopy neural network} model constructed for the input graph as shown in the left plot. In the example, the input graph can be represented as $G = (\mathcal{V}, \mathcal{E})$, where $\mathcal{V} = \{v_1, v_2, \cdots, v_6\}$ and $\mathcal{E} = \{(v_1, v_2), (v_1, v_3), (v_1, v_4), (v_2, v_3),(v_3, v_4),(v_4, v_5),(v_5, v_6)\}$. The constructed \textit{deep loopy neural network} model involves $k$ layers and the intra-layer neuron connections mainly exist in these hidden layers respectively.

Formally, given the node input feature vector $\mb{x}_i$ for node $v_i$, we can represents its hidden states at different hidden layers and the output label as follows:
\begin{equation}
\begin{cases}
\mb{h}^{(1)}_i &= \sigma \big(\mb{W}^{x} \mb{x}_i + \mb{b}^{x}  + \sum_{v_j \in \Gamma(v_i)} (\mb{W}^{h_1}  \mb{h}_j + \mb{b}^{h_1} ) \big),\\
&\cdots\\
\mb{h}^{(k)}_i &= \sigma \big(\mb{W}^{h_{k-1}, h_k} \mb{h}^{(k-1)}_i + \mb{b}^{h_{k-1}, h_k} + \sum_{v_j \in \Gamma(v_i)} (\mb{W}^{h_k} \mb{h}_j + \mb{b}^{h_k} ) \big),\\
{\mb{y}}_i &= \sigma(\mb{W}^{y} \mb{h}^{(k)}_i + \mb{b}^{y}),
\end{cases}
\end{equation}
where set $\Gamma(v_i) = \{v_j | v_j \in \mathcal{V} \land (v_i, v_j) \in \mathcal{E}\}$ denotes the neighbors of node $v_i$ in the input graph.

Compared against the true label vector of nodes in the network, e.g., $\hat{\mb{y}}_i$ for $v_i \in \mathcal{V}$, we can represent the introduced loss by the model on the input graph as
\begin{equation}
E(G) = \sum_{v_i \in \mathcal{V}} E(v_i) = \sum_{v_i \in \mathcal{V}} loss( \hat{\mb{y}}_i, {\mb{y}}_i ),
\end{equation}
where different loss functions can be adopted here, e.g., \textit{mean square loss} or \textit{cross-entropy loss}.

By minimizing the loss function, we will be able to learn the variables involved in the model. By this context so far, most of the deep neural network model training is based on the error back propagation algorithm. However, applying the error back propagation algorithm to train the \textit{deep loopy neural network} model will encounter great challenges due to the extensive variable dependence relationships created by the loops, which will be illustrated in great detail later.

\begin{figure*}[!t]
 \centering    
 \begin{minipage}[l]{0.8\columnwidth}
  \centering
    \includegraphics[width=1.0\textwidth]{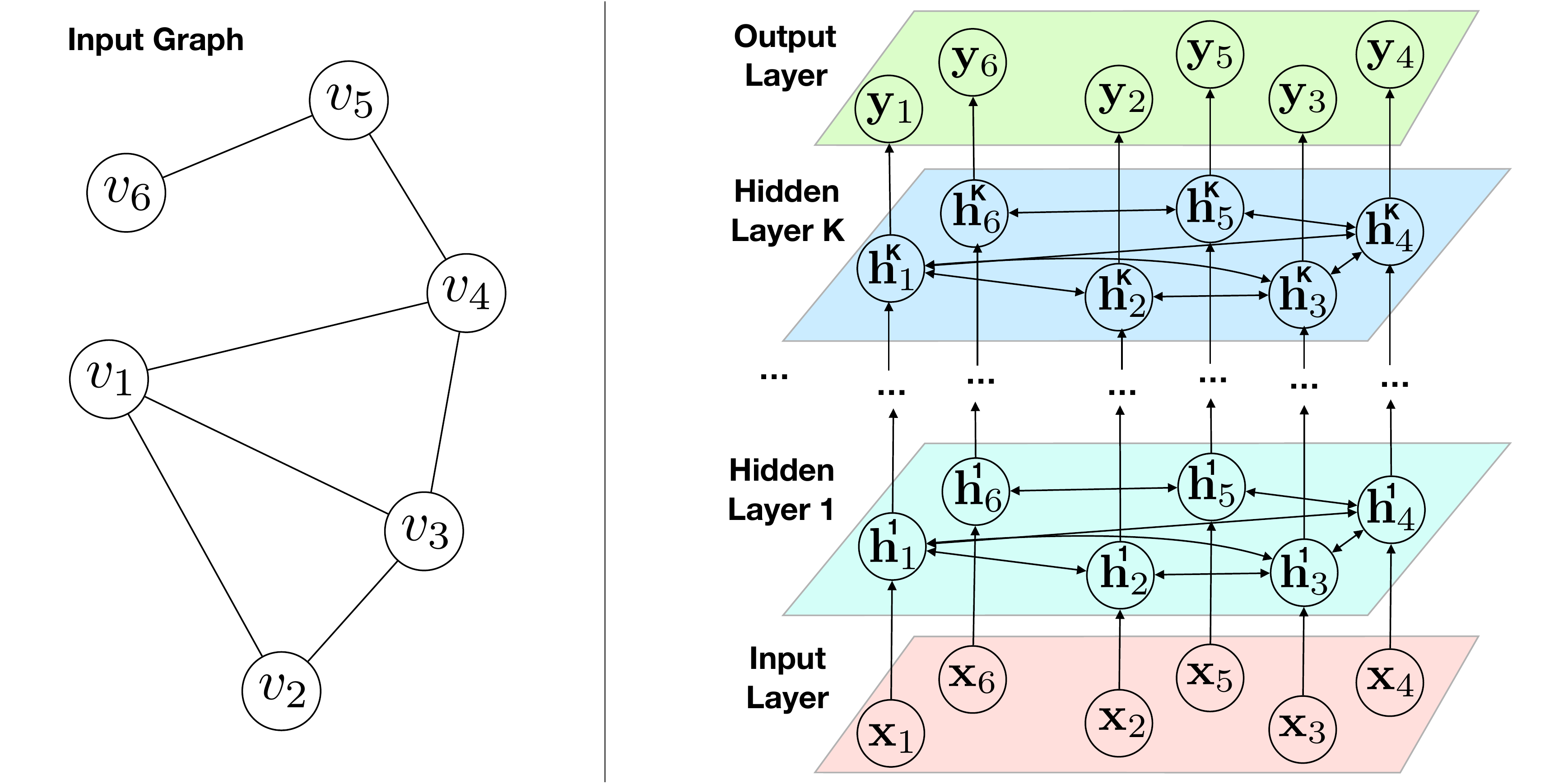}
 \end{minipage}
\caption{Overall Architecture of Deep Loopy Neural Network Model.}\label{fig:framework}
\end{figure*}

The following part of this paper is organized as follows. In Section~\ref{sec:relatedwork}, we will introduce the existing related works to this paper. We will analyze the challenges in learning \textit{deep loop neural networks} with error back-propagation algorithm in Section~\ref{sec:analysis}, and a new learning algorithm will be introduced in Section~\ref{sec:method}. Extensive numerical experiments will be provided to evaluate the model performance in Section~\ref{sec:experiment}, and finally we will conclude this paper in Section~\ref{sec:conclusion}.

\section{Related Works} \label{sec:relatedwork}

Two research topics are closely related to this paper, including deep learning and network representation learning, and we will provide a brief overview of the existing papers published on these two topics as follows.

\noindent \textbf{Deep Learning Research and Applications}: The essence of deep learning is to compute hierarchical features or representations of the observational data \cite{GBC16, LBH15}. With the surge of deep learning research and applications in recent years, lots of research works have appeared to apply the deep learning methods, like deep belief network \cite{HOT06}, deep Boltzmann machine \cite{SH09}, deep neural network \cite{J02, KSH12} and deep autoencoder model \cite{VLLBM10}, in various applications, like speech and audio processing \cite{DHK13, HDYDMJSVNSK12}, language modeling and processing \cite{ASKR12, MH09}, information retrieval \cite{H12, SH09}, objective recognition and computer vision \cite{LBH15}, as well as multimodal and multi-task learning \cite{WBU10, WBU11}.

\noindent \textbf{Network Embedding}: Network embedding has become a very hot research problem recently, which can project a graph-structured data to the feature vector representations. In graphs, the relation can be treated as a translation of the entities, and many translation based embedding models have been proposed, like TransE \cite{BUGWY13}, TransH \cite{WZFC14} and TransR \cite{LLSLZ15}. In recent years, many network embedding works based on random walk model and deep learning models have been introduced, like Deepwalk \cite{PAS14}, LINE \cite{TQWZYM15}, node2vec \cite{GL16}, HNE \cite{CHTQAH15} and DNE \cite{WCZ16}. Perozzi et al. extends the word2vec model \cite{MSCCD13} to the network scenario and introduce the Deepwalk algorithm \cite{PAS14}. Tang et al. \cite{TQWZYM15} propose to embed the networks with LINE algorithm, which can preserve both the local and global network structures. Grover et al. \cite{GL16} introduce a flexible notion of a node's network neighborhood and design a biased random walk procedure to sample the neighbors. Chang et al. \cite{CHTQAH15} learn the embedding of networks involving text and image information. Chen et al. \cite{CS16} introduce a task guided embedding model to learn the representations for the author identification problem. 
\section{Learning Challenges Analysis of Deep Loopy Neural Network}\label{sec:analysis}

In this part, we will analyze the challenges in training the \textit{deep loop neural network} with traditional error back propagation algorithm. To simplify the settings, we assume the \textit{deep loop neural network} has only one hidden layer, i.e., $k=1$. According to the definition of the \textit{deep loopy neural network} model provided in Section~\ref{sec:intro}, we can represent the inferred labels for graph nodes, e.g., $v_i \in \mathcal{V}$, as vector ${\mb{y}}_i = [{y}_{i,1}, {y}_{i,2}, \cdots, {y}_{i,d}]^\top$, and its $j_{th}$ entry $\mb{y}_{i}(j)$ (or ${y}_{i,j}$) can be denoted as \begingroup\makeatletter\def\f@size{9.5}\check@mathfonts
\begin{equation}
\begin{cases}
\mb{y}_{i}(j) &\hspace{-10pt}= \sigma \big( \sum_{e = 1}^{m} \mb{W}^{y}(e, j) \cdot \mb{h}_{i}(e) + \mb{b}^{y}(j) \big),\\
 \mb{h}_{i}(e) &\hspace{-10pt}= \sigma \big( \sum_{c = 1}^n \mb{W}^{x}(c,e) \cdot \mb{x}_i(c) + \mb{b}^{x}(e) + \sum_{v_p \in \Gamma(v_i)} \sum_{f = 1}^m \mb{W}^h(f,e) \cdot \mb{h}_p(f) + \mb{b}^h(e) \big).
\end{cases}
\end{equation}\endgroup

Here, we will use the \textit{mean square error} as an example of the loss function, and the loss introduced by the model for node $v_i \in \mathcal{V}$ compared with the ground truth label $\hat{\mb{y}}_i$ can be represented as
\begin{equation}
E(v_i) = \frac{1}{2} \left\| \mb{y}_i - \hat{\mb{y}}_i \right\|_2^2 = \frac{1}{2} \sum_{j = 1}^d \left(\mb{y}_i(j) - \hat{\mb{y}}_i(j) \right)^2.
\end{equation}

\subsection{Learning Output Layer Variables}

The variables involved in the \textit{deep loopy neural network} model can be learned by error back propagation algorithm. For instance, here we will use SGD as the optimization algorithm. Given the node $v_i$ and its introduced error function $E(v_i)$, we can represent the updating equation for the variables in the output layer, e.g., $\mb{W}^y(e,j)$ and $\mb{b}^y(j)$, as follows:
\begin{equation}
\begin{cases}
\mb{W}^y_{\tau}(e,j) &= \mb{W}^y_{\tau-1}(e,j) - \eta^w_{\tau} \cdot \frac{\partial E(v_i)}{\partial \mb{W}^y_{\tau-1}(e,j)}  ,\\
\mb{b}^y_{\tau}(j) &= \mb{b}^y_{\tau-1}(j) - \eta^b_{\tau} \cdot \frac{\partial E(v_i)}{\partial \mb{b}^y_{\tau-1}(j) } .
\end{cases}
\end{equation}
where $\eta^w_{\tau}$ and $\eta^b_{\tau}$ denote the learning rates in updating $\mb{W}^y$ and $\mb{b}^y$ at iteration $\tau$ respectively. 

According to the derivative chain rule, we can represent the partial derivative terms $\frac{\partial E(v_i)}{\partial \mb{W}^y_{\tau-1}(e,j)}$ and $\frac{\partial E(v_i)}{\partial \mb{b}^y_(j) }$ as follows respectively:
\begin{equation}
\begin{cases}
\frac{\partial E(v_i)}{\partial \mb{W}^y(e,j)} &=  \frac{\partial E(v_i)}{\partial \mb{y}_i(j)} \cdot \frac{\partial \mb{y}_i(j)}{\partial \mb{z}^h_i(j)} \cdot \frac{\partial \mb{z}^h_i(j)}{\partial  \mb{W}^y(e,j)}= \big( \mb{y}_i(j) - \hat{\mb{y}}_i(j) \big) \cdot \mb{y}_i(j) \big( 1- \mb{y}_i(j) \big) \cdot \mb{h}_i(e),\vspace{5pt}\\
\frac{\partial E(v_i)}{\partial \mb{b}^y(j) } &=  \frac{\partial E(v_i)}{\partial \mb{y}_i(j)} \cdot \frac{\partial \mb{y}_i(j)}{\partial \mb{z}^h_i(j)} \cdot \frac{\partial \mb{z}^h_i(j)}{\partial  \mb{b}^y(j) }= \big( \mb{y}_i(j) - \hat{\mb{y}}_i(j) \big) \cdot \mb{y}_i(j) \big( 1- \mb{y}_i(j) \big) \cdot 1,
\end{cases}
\end{equation}
where term $\mb{z}^h_i(j) = \sum_{e = 1}^{m} \mb{W}^{y}(e, j) \cdot \mb{h}_{i}(e) + \mb{b}^{y}(j)$.

\subsection{Learning Hidden Layer and Input Layer Variables}

Meanwhile, when updating the variables in the hidden and output layers, we will encounter great challenges in computing the partial derivatives of the error function regarding these variables. Given two connected nodes $v_i, v_j \in \mathcal{V}$, where $(v_i, v_j) \in \mathcal{E}$, from the graph, we have the representation of their hidden state vectors $\mb{h}_i$ and $\mb{h}_j$ as follows:
\begin{align}
\mb{h}_i &=  \sigma \big(\mb{W}^x \mb{x}_i + \mb{b}^x + \mb{W}^h \mb{h}_j + \mb{b}^h + \sum_{v_k \in \Gamma(v_i) \setminus \{v_j\}} (\mb{W}^h \mb{h}_{k} + \mb{b}^h) \big),\\
\mb{h}_j &=   \sigma \big(\mb{W}^x \mb{x}_j + \mb{b}^x + \mb{W}^h \mb{h}_i + \mb{b}^h + \sum_{v_k' \in \Gamma(v_j) \setminus \{v_i\}} (\mb{W}^h \mb{h}_{k'} + \mb{b}^h) \big).
\end{align} 
We can observe that $\mb{h}_i$ and $\mb{h}_j$ co-rely on each other in the computation, whose representation will be involved in an infinite recursive definition. When we compute the partial derivative of error function regarding variable $\mb{W}^x$, $\mb{b}^x$ or $\mb{W}^h$, $\mb{b}^h$, we will have an infinite partial derivative sequence involving $\mb{h}_i$ and $\mb{h}_j$ according to the chair rule. The problem will be much more serious for connected graphs, as illustrated by the following theorem.

\begin{theo}\label{theo:dependency}
Let $G$ denote the input graph. If graph $G$ is connected (i.e., there exist a path connecting any pairs of nodes in the graph), for any node $v_i$ in the graph, its hidden state vector $\mb{h}_i$ will be involved in the hidden state representation of all the other nodes in the graph.
\end{theo}

\begin{proof}
The Theorem can be proved via contradiction. 

Here, we assume hidden state vector $\mb{h}_i$ is not involved in the hidden representation of a certain node $v_j$, i.e., $\mb{h}_j$. Formally, given the neighbor set $\Gamma(v_j)$ of node $v_j$ in the network, we know that vector $\mb{h}_j$ is defined based on the hidden state vectors $\{\mb{h}_k\}_{v_k \in \Gamma(v_i)}$. Therefore, we can assert that vector $\mb{h}_i$ should also not be involved in the representation of all the nodes in $\Gamma(v_j)$. Viewed in this perspective, we can also show that $\mb{h}_i$ is not involved in the hidden state representations of the 2-hop neighbors of node $v_j$, as well as the 3-hop, and so forth. 

Considering that graph $G$ is connected, we know there should exist a path of length $p$ connecting $v_j$ with $v_i$, i.e., $v_i$ will be a $p$-hop neighbor of node $v_j$, which will contract the claim $\mb{h}_i$ is not involved in the hidden state representations of the $p$-hop neighbors of node $v_j$. Therefore, the assumption we make at the beginning doesn't hold and we will prove the theorem.
\end{proof}

According to the above theorem, when we try to compute the derivative of the error function regarding variable $\mb{W}^x$, we will have
\begin{equation}
\frac{\partial E(v_i)}{\partial \mb{W}^x} = \frac{\partial E(v_i)}{\partial \mb{y}_i} \cdot \frac{\partial \mb{y}_i}{\partial \mb{z}^h_i} \cdot \frac{\partial \mb{z}^h_i} {\partial \mb{h}_i} \cdot \bigg( \frac{\partial \mb{h}_i} {\partial \mb{W}^x} + \sum_{v_j \in \mathcal{V}} \frac{\partial \mb{h}_i} {\partial \mb{h}_j} \cdot  \Big( \frac{\partial \mb{h}_j} {\partial \mb{W}^x}  + \sum_{v_k \in \mathcal{V}} \frac{\partial \mb{h}_j} {\partial \mb{h}_k} \cdot \big( \cdots  \big)  \Big) \bigg)
\end{equation}
Due to the recursive definition of the hidden state vector $\{\mb{h}_i\}_{v_i \in \mathcal{V}}$ in the network, the partial derivative of term $\frac{\partial E(v_i)}{\partial \mb{W}^x}$ will be impossible to compute mathematically. The partial derivative sequence will extend to an infinite length according to Theorem~\ref{theo:dependency}. Similar phenomena can be observed when computing the partial derivative of the error function regarding variables $\mb{b}^x$ or $\mb{W}^h$, $\mb{b}^h$.

\section{Proposed Method}\label{sec:method}

To resolve the challenges introduced in the previous section, in this part, we will propose an approximation algorithm to learn the \textit{deep loopy neural network} model. Given the complete model architecture (which is also in a graph shape), to differentiate it from the input graph data, we will name it as the \textit{model graph} formally, the neuron vectored involved in which are called the neuron nodes by default. From the \textit{model graph}, we propose to extract a set of tree structured model diagrams rooted at certain neuron states in the model graph. Model learning will be mainly performed on these extracted rooted trees instead.

\subsection{g-Hop Model Subgraph and Rooted Spanning Tree Extraction}

Given the \textit{deep loopy neural network} model graph, involving the input, output and hidden state variables and the projections parameterized by the variables to be learned, we propose to extract a set of \textit{g-hop subgraph} (\textit{g-subgraph}) around certain target neuron node from it, where all the neuron nodes involved are within $g$ hops from the target neuron node.

\begin{defn}
(\textbf{g-hop subgraph}): Given the \textit{deep loopy neural network} model graph $\mathcal{G} = (\mathcal{N}, \mathcal{L})$ and a target neuron node, e.g., ${\mb{n}}\in \mathcal{N}$, we can represent the extracted \textit{g-subgraph} rooted at ${\mb{n}}$ as $\mathcal{G}_{\mb{n}} = (\mathcal{N}_{\mb{n}}, \mathcal{L}_{\mb{n}})$, where $\mathcal{N}_{\mb{n}} \subset \mathcal{N}$ and $\mathcal{L}_{\mb{n}} \subset \mathcal{L}$. For all the nodes in set $\mathcal{N}_{\mb{n}}$, formed by links in $\mathcal{L}_{\mb{y}_i}$, there exist a directed path of length less than $g$ connecting the neuron node ${\mb{n}}$ with them.
\end{defn}

\begin{figure*}[!t]
 \centering    
 \begin{minipage}[l]{1.0\columnwidth}
  \centering
    \includegraphics[width=1.0\textwidth]{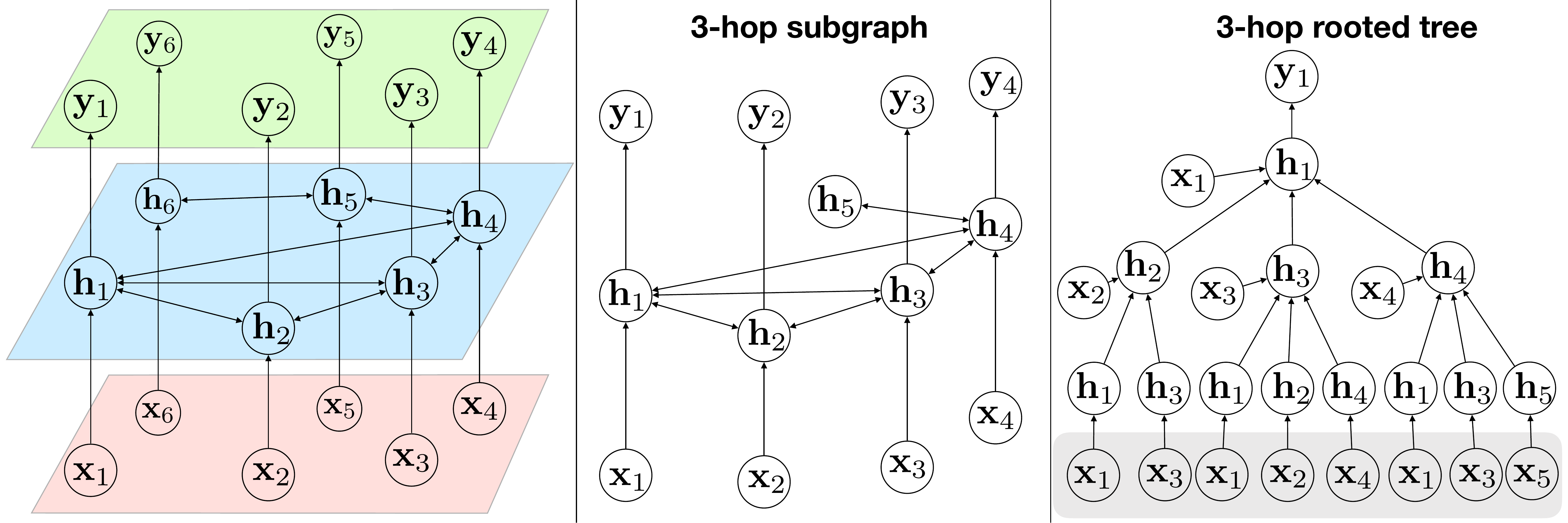}
 \end{minipage}
\caption{Structure of 3-Hop Subgraph and Rooted Spanning Tree at $\mb{y}_1$. (The bottom feature nodes in the gray component attached to the tree leaves are append for model learning purposes.)}\label{fig:example}
\end{figure*}

For instance, from the \textit{deep loopy neural network} model graph (with $1$ hidden layer) as shown in the left plot of Figure~\ref{fig:example}, we can extract a \textit{3-subgraph} rooted at neuron node $\mb{y}_1$ as shown in the central plot. In the sub-graph, it contains the feature, label and hidden state vectors of nodes $v_2, v_3, v_4$ as well as the hidden state vector of $v_5$, whose feature and label vectors are not included since they are 4-hops away from $\mb{h}_1$. 

In the model graph, the link direction denote the forward propagation direction. In the model learning process, the error information will propagate from the output layer, i.e., the label neuron nodes, backward to the hidden state and input neuron nodes along the reversed direction of these links. For instance, given the target node $\mb{y}_1$, its error can be propagated to $\mb{h}_1, \mb{h}_2, \cdots, \mb{h}_5$ and $\mb{x}_1, \mb{x}_2, \mb{x}_3, \mb{x}_4$, but cannot reach $\mb{y}_2$, $\mb{y}_3$ and $\mb{y}_4$. Meanwhile, to resolve the recursive partial derivative problem introduced in the previous section, in this paper, we propose to further extract a \textit{g-hop rooted spanning tree} (\textit{g-tree}) from the \textit{g-subgraph}. The extracted \textit{g-tree} will be acyclic and all the involved links are from the children nodes to the parent nodes only.

\begin{defn}
(\textbf{g-hop rooted spanning tree}): Given the extracted \textit{g-subgraph} around the target neuron node $\mb{n}$, i.e., $\mathcal{G}_{\mb{n}} = (\mathcal{N}_{\mb{n}}, \mathcal{L}_{\mb{n}})$, we can represent its \textit{g-tree} as $\mathcal{T}_{\mb{n}} = (\mathcal{S}_{\mb{n}}, \mathcal{R}_{\mb{n}}, \mb{n})$, where $\mb{n}$ is the root, sets $\mathcal{S}_{\mb{n}} \subset \mathcal{N}_{\mb{n}}$ and $\mathcal{R}_{\mb{n}} \subset \mathcal{L}_{\mb{n}}$. All the links in $\mathcal{P}_{\mb{n}}$ are pointing from the children nodes to the parent nodes.
\end{defn}

For instance, in the right plot of Figure~\ref{fig:example}, we display an example of the extracted \textit{3-tree} rooted at neuron node $\mb{y}_1$. From the plot, we observe that all the label vectors are removed, since there exist no directed edges from them to the root node. Among all the nodes, $\mb{h}_1$ is $1$-hop away from $\mb{y}_1$, $\mb{x}_1$, $\mb{h}_2$, $\mb{h}_3$ and $\mb{h}_4$ are $2$-hop away from $\mb{y}_1$, and all the remaining nodes are $3$ hops away from $\mb{h}_1$. Given any two pair of connected nodes in the \textit{3-tree}, e.g., $\mb{x_1}$ and $\mb{y}_1$ (or $\mb{h}_4$ and $\mb{h}_5$), the edges connecting them clearly indicate the variables to be learned via error back propagation across them. When learning the \textit{loopy neural network}, instead of using the whole network, we propose to back propagate the errors from the root, e.g., $\mb{y}_1$, to the remaining nodes in the extracted \textit{3-tree}.

\begin{theo}
Given $g = \infty$, the learning process based on \textit{g-tree} will be identical as learning based on the original network. Meanwhile, in the case when $g$ is a finite number and $g \ge \mbox{diameter}(\mathcal{G})$, the \textit{g-tree} of any nodes will actually cover all the neuron nodes and variables to be learned.
\end{theo}

The proof to the above theorem will not be introduced here due to the limited space. Generally, larger $g$ can preserve more complete network structure information. However, on the other hand, as $g$ increases, the paths involved in the \textit{g-tree} from the leaf node to the root node will also be longer, which may lead to the gradient vanishing/exploding problem as well \cite{PMB13}. 

\subsection{Rooted Spanning Tree based Learning Algorithm for Loopy Neural Network Model}

Based on the extracted \textit{g-tree}, the \textit{deep loopy neural network} model can be effectively trained, and in this section we will introduce the general learning algorithm in detail. Formally, let $\mathcal{T}_{\mb{n}} = (\mathcal{S}_{\mb{n}}, \mathcal{R}_{\mb{n}}, \mb{n})$ denote an extracted spanning tree rooted at neuron node $\mb{n}$. Based on the errors computed on node $\mb{n}$ (if $\mb{n}$ is in the output layer) or the errors propagated to $\mb{n}$, we can further back propagate the errors to the remaining nodes in $\mathcal{S}_{\mb{n}} \setminus \{\mb{n}\}$.

Formally, from the spanning tree $\mathcal{T}_{\mb{n}} = (\mathcal{S}_{\mb{n}}, \mathcal{R}_{\mb{n}}, \mb{n})$, we can define the set of variables involved as $\mathcal{W}$, which can be indicated by the links in set $\mathcal{R}_{\mb{n}}$. For instance, given the \textit{3-tree} as shown in Figure~\ref{fig:example}, we know that there exist three types of variables involved in the tree diagram, where the variable set $\mathcal{W} = \{\mb{W}^x, \mb{b}^x, \mb{W}^y, \mb{b}^y, \mb{W}^h, \mb{b}^h\}$. For the spanning trees extracted from deeper neural network models, the variable type set will be much larger. Meanwhile, given a random node $\mb{m} \in \mathcal{S}_{\mb{n}}$, we will use notation $\mathcal{T}_{\mb{n}}(\mb{m})$ to denote a subtree of $\mathcal{T}_{\mb{n}}$ rooted at $\mb{m}$, and notation $\Gamma(\mb{m})$ to represent the children neuron nodes of $\mb{m}$ in $\mathcal{T}_{\mb{n}}$. Furthermore, for simplicity, we will use notation $\mb{W} \in \mathcal{T}$ to denote that variable $\mb{W} \in \mathcal{W}$ is involved in the tree or sub-tree $\mathcal{T}$.

Given the spanning tree $\mathcal{T}_{\mb{n}}$ together with the errors $E(\mb{n})$ computed at $\mb{n}$ (or propagated to $\mb{n}$), regarding a variable $\mb{W} \in \mathcal{W}$, we can first define a basic learning operation between neuron nodes $\mb{x}, \mb{y} \in \mathcal{S}_{\mb{n}}$ (where $\mb{y} \in \Gamma(\mb{x})$) as follows:
\begin{equation}
\mbox{\textsc{Prop}}(\mb{x}, \mb{y}; \mb{W}) = 
\begin{cases} 
\mathbbm{1}\big(\mb{W} \in \mathcal{T}_{\mb{n}}(\mb{m})\big) \frac{\partial \mb{x}}{\partial \mb{W}}, & \mbox{ if $\mb{y}$ is a leaf node};\\
\mathbbm{1}\big(\mb{W} \in \mathcal{T}_{\mb{n}}(\mb{m})\big) \frac{\partial \mb{x}}{\partial \mb{y}} \cdot \Big( \sum_{\mb{z} \in \Gamma(\mb{y})} \mbox{\textsc{Prop}}(\mb{y}, \mb{z}; \mb{W}) \Big), & \mbox{ otherwise}.
\end{cases}
\end{equation}

Based on the above operation, we can represent the partial derivative of the error function $E(\mb{n})$ regarding variable $\mb{W}$ as:
\begin{align}
\frac{\partial E(\mb{n})}{\partial \mb{W}} = \frac{\partial E(\mb{n})}{\partial \mb{n}} \cdot \big(\sum_{\mb{m} \in \Gamma(\mb{n})} \mbox{\textsc{Prop}}(\mb{n}, \mb{m}; \mb{W})  \big)
\end{align}

\begin{theo}
Formally, given a $g$-hop rooted spanning tree $\mathcal{T}_{\mb{n}}$, the $\mbox{\textsc{Prop}}(\mb{x}, \mb{y}; \mb{W})$ operation will be called at most $\frac{ (d_{\max} + 1)^{k+1} - d_{\max} - 1}{d_{\max}}$ times in computing the partial derivative term $\frac{\partial E(\mb{n})}{\partial \mb{W}}$, where $d_{\max}$ denotes the largest node degree in the original input graph data $G$.
\end{theo}

\begin{proof}
Operation $\mbox{\textsc{Prop}}(\mb{x}, \mb{y}; \mb{W})$ will be called for each node in the spanning tree $\mathcal{T}_{\mb{n}}$ except the root node. Formally, given a max node degree $d_{\max}$ in the original graph, the largest number of children node connected to a random neuron node in the spanning tree will be $d_{\max} + 1$ ($+1$ because of the connection from the lower layer neuron node). The maximum number of nodes (except the root node) in the spanning tree of depth $g$ can be denoted as the sum 
\begin{equation}
(d_{\max} + 1) + (d_{\max} + 1)^2 +  (d_{\max} + 1)^3 + \cdots + (d_{\max} + 1)^g,
\end{equation} 
which equals to $\frac{ (d_{\max} + 1)^{g+1} - d_{\max} - 1}{d_{\max}}$ as indicated in the theorem.
\end{proof}

Considering that in the model, the hidden neuron nodes are computed based on the lower-level neurons and they don't have input representations. To make the model learnable, as indicated by the bottom gray layer below the \textit{g-tree} in Figure~\ref{fig:example}, we propose to append the input feature representations to the \textit{g-tree} leaf nodes, based on which we will be able to compute the node hidden states. The learning process will involve several epochs until convergence, where each epoch will enumerate all the nodes in the graph once. To further boost the convergence rate, several latest optimization algorithms, e.g., Adam \cite{KB14}, can be adopted to replace traditional SGD in updating the model variables.

\section{Numerical Experiments}\label{sec:experiment}

To test the effectiveness of the proposed \textit{deep loopy neural network} and the learning algorithm, extensive numerical experiments have been done on several frequently used graph benchmark datasets, including two social networks: Foursquare and Twitter, as well as two knowledge graphs: Douban and IMDB, which will be introduced in this section.

\begin{table*}[t]
\caption{Experimental Results on the Douban Graph Dataset.}
\scriptsize
\centering
\setlength{\tabcolsep}{4pt}
\begin{tabular}{c  cccc  }
\toprule
&\multicolumn{4}{c}{Douban} \\
\cmidrule(lr){2-5}
methods &{MSE} &{MAE} &{LRS} &{{ avg. rank}}  \\
\midrule 
{{\our}}	
&\textbf{0.035}$\pm$\textbf{0.0} {\tiny {\color{blue} (1)}}
&\textbf{0.067}$\pm$\textbf{0.001} {\tiny {\color{blue} (1)}}
&\textbf{0.102}$\pm$\textbf{0.002} {\tiny {\color{blue} (1)}}
&{\color{blue} (1)}\\

\midrule 
{\deepwalk} \cite{PAS14}	
&0.05$\pm$0.0 {\tiny {\color{blue} (7)}}
&0.097$\pm$0.0 {\tiny {\color{blue} (7)}}
&0.129$\pm$0.003 {\tiny {\color{blue} (8)}}
&{\color{blue} (7)}\\

\midrule 
{\walklets} \cite{PKS16}		
&0.046$\pm$0.001 {\tiny {\color{blue} (3)}}
&0.087$\pm$0.001 {\tiny {\color{blue} (3)}}
&0.108$\pm$0.001 {\tiny {\color{blue} (3)}}
&{\color{blue} (2)}\\

\midrule 
{\linemodel} \cite{TQWZYM15}	
&0.048$\pm$0.0 {\tiny {\color{blue} (6)}}
&0.091$\pm$0.001 {\tiny {\color{blue} (6)}}
&0.12$\pm$0.003 {\tiny {\color{blue} (6)}}
&{\color{blue} (6)}\\

\midrule 
{\hpe} \cite{CTLY16}		
&0.046$\pm$0.001 {\tiny {\color{blue} (3)}}
&0.077$\pm$0.001 {\tiny {\color{blue} (2)}}
&0.111$\pm$0.002 {\tiny {\color{blue} (4)}}
&{\color{blue} (2)}\\

\midrule 
{\app} \cite{ZLLLG17}		
&0.05$\pm$0.0 {\tiny {\color{blue} (8)}}
&0.099$\pm$0.0 {\tiny {\color{blue} (8)}}
&0.127$\pm$0.002 {\tiny {\color{blue} (7)}}
&{\color{blue} (8)}\\

\midrule 
{\mf} \cite{CHHH11}		
&0.045$\pm$0.0 {\tiny {\color{blue} (2)}}
&0.089$\pm$0.001 {\tiny {\color{blue} (5)}}
&\textbf{0.102}$\pm$\textbf{0.001} {\tiny {\color{blue} (1)}}
&{\color{blue} (4)}\\

\midrule 
{\bpr} \cite{RFGS09}		
&0.046$\pm$0.0 {\tiny {\color{blue} (3)}}
&0.089$\pm$0.0 {\tiny {\color{blue} (4)}}
&0.114$\pm$0.002 {\tiny {\color{blue} (5)}}
&{\color{blue} (5)}\\

\bottomrule
\end{tabular}\label{tab:douban_result}
\vspace{-10pt}
\end{table*}
\begin{table*}[t]
\caption{Experimental Results on the IMDB Graph Dataset.}
\scriptsize
\centering
\setlength{\tabcolsep}{4pt}
\begin{tabular}{c  cccc   }
\toprule
&\multicolumn{4}{c}{IMDB} \\
\cmidrule(lr){2-5}
methods & {MSE} &{MAE} &{LRS} &{{ avg. rank}} \\
\midrule 
{{\our}}	
&\textbf{0.046}$\pm$\textbf{0.0} {\tiny {\color{blue} (1)}}
&\textbf{0.090}$\pm$\textbf{0.001} {\tiny {\color{blue} (1)}}
&\textbf{0.116}$\pm$\textbf{0.002} {\tiny {\color{blue} (1)}}
&{\color{blue} (1)}\\

\midrule 
{\deepwalk} \cite{PAS14}	
&0.065$\pm$0.0 {\tiny {\color{blue} (7)}}
&0.128$\pm$0.001 {\tiny {\color{blue} (8)}}
&0.187$\pm$0.002 {\tiny {\color{blue} (7)}} 
&{\color{blue} (7)}\\

\midrule 
{\walklets} \cite{PKS16}		
&0.057$\pm$0.0 {\tiny {\color{blue} (4)}}
&0.112$\pm$0.001 {\tiny {\color{blue} (3)}}
&0.139$\pm$0.003 {\tiny {\color{blue} (4)}}
&{\color{blue} (3)}\\

\midrule 
{\linemodel} \cite{TQWZYM15}	
&0.061$\pm$0.0 {\tiny {\color{blue} (6)}}
&0.121$\pm$0.001 {\tiny {\color{blue} (7)}}
&0.169$\pm$0.004 {\tiny {\color{blue} (6)}}
&{\color{blue} (6)}\\

\midrule 
{\hpe} \cite{CTLY16}		
&0.055$\pm$0.0 {\tiny {\color{blue} (2)}}
&0.107$\pm$0.001 {\tiny {\color{blue} (2)}}
&0.127$\pm$0.003 {\tiny {\color{blue} (2)}}
&{\color{blue} (2)}\\

\midrule 
{\app} \cite{ZLLLG17}		
&0.066$\pm$0.0 {\tiny {\color{blue} (8)}}
&0.13$\pm$0.001 {\tiny {\color{blue} (6)}}
&0.193$\pm$0.002 {\tiny {\color{blue} (8)}}
&{\color{blue} (7)}\\

\midrule 
{\mf} \cite{CHHH11}		
&0.055$\pm$0.0 {\tiny {\color{blue} (2)}}
&0.112$\pm$0.001 {\tiny {\color{blue} (3)}}
&0.129$\pm$0.003 {\tiny {\color{blue} (3)}}
&{\color{blue} (3)}\\

\midrule 
{\bpr} \cite{RFGS09}		
&0.058$\pm$0.0 {\tiny {\color{blue} (5)}}
&0.117$\pm$0.0 {\tiny {\color{blue} (5)}}
&0.157$\pm$0.002 {\tiny {\color{blue} (5)}}
&{\color{blue} (5)}\\

\bottomrule
\end{tabular}\label{tab:imdb_result}
\vspace{-10pt}
\end{table*}

\subsection{Experimental Settings}

The basic statistics information about the graph datasets used in the experiments are as follows:
\begin{itemize}
\item \textit{Douban}: Number of Nodes: $11,297$; Number of Links: $753,940$. 
\item \textit{IMDB}: Number of Nodes: $13,896$; Number of Links: $1,404,202$. 
\item \textit{Foursquare}: Number of Nodes: $5,392$; Number of Links: $111,852$. 
\item \textit{Twitter}: Number of Nodes: $5,120$; Number of Links: $261,151$. 
\end{itemize}
Both Foursquare and Twitter are online social networks. The nodes and links involved in Foursquare and Twitter denote the users and their friendships respectively. Douban and IMDB are two movie knowledge graphs, where the nodes denote the movies. Based on the movie cast information, we construct the connections among the movies, where two movies can be linked if they share a common cast member. Besides the network structure information, a group of node attributes can also be collected for users and movies, which cover the user and movie basic profile information respectively. Based on the attribute information, a set of features can be extracted as the node input feature information. In the experiments, we use the movie genre and user hometown state as the labels. These labeled nodes are partitioned into training and testing sets via $5$-fold cross validation: $4$-fold for training, and $1$-fold for testing.  

Meanwhile, to demonstrate the advantages of the learned \textit{deep loopy neural network} against the other existing network embedding models, many baseline methods are compared with \textit{deep loopy neural network} in the experiments. The baseline methods cover the state-of-the-art methods on graph data representation learning published in recent years, which include {\deepwalk} \cite{PAS14}, {\walklets} \cite{PKS16}, {\linemodel} (Large-scale Information Network Embedding) \cite{TQWZYM15}, {\hpe} (Heterogeneous Preference Embedding) \cite{CTLY16}, {\app} (Asymmetric Proximity Preserving graph embedding) \cite{ZLLLG17}, {\mf} (Matrix Factorization) \cite{CHHH11} and {\bpr} (Bayesian Personalized Ranking) \cite{RFGS09}. For these network representation baseline methods, based on the learned representation features, we will further train a MLP (multiple layer perceptron) as the classifier. By comparing the inference results by the models on the testing set with the ground truth label vectors, we will use MSE (mean square error), MAE (mean absolute error) and LRS (label ranking loss) as the evaluation metrics. Detailed experimental results will be provided in the following subsection.

\begin{table*}[t]
\caption{Experimental Results on the Foursquare Social Network Dataset.}
\scriptsize
\centering
\setlength{\tabcolsep}{4pt}
\begin{tabular}{c  cccc   }
\toprule
&\multicolumn{4}{c}{Foursquare Network} \\
\cmidrule(lr){2-5}
methods &{MSE} &{MAE} &{LRS} &{{ avg. rank}} \\
\midrule 
{{\our}}	
&0.002$\pm$0.001 {\tiny {\color{blue} (3)}}
&\textbf{0.003}$\pm$\textbf{0.001} {\tiny {\color{blue} (1)}}
&\textbf{0.121}$\pm$\textbf{0.001} {\tiny {\color{blue} (1)}}
&{\color{blue} (1)}\\

\midrule 
{\deepwalk} \cite{PAS14}	
&0.002$\pm$0.001 {\tiny {\color{blue} (3)}}
&0.004$\pm$0.001 {\tiny {\color{blue} (4)}}
&0.149$\pm$0.007 {\tiny {\color{blue} (3)}}
&{\color{blue} (3)}\\

\midrule 
{\walklets} \cite{PKS16}		
&\textbf{0.001}$\pm$\textbf{0.001} {\tiny {\color{blue} (1)}}
&0.004$\pm$0.001 {\tiny {\color{blue} (4)}}
&0.216$\pm$0.0016 {\tiny {\color{blue} (5)}}
&{\color{blue} (3)}\\

\midrule 
{\linemodel} \cite{TQWZYM15}	
&0.002$\pm$0.001 {\tiny {\color{blue} (3)}}
&\textbf{0.003}$\pm$\textbf{0.001} {\tiny {\color{blue} (1)}}
&0.309$\pm$0.0017 {\tiny {\color{blue} (8)}}
&{\color{blue} (6)}\\

\midrule 
{\hpe} \cite{CTLY16}		
&\textbf{0.001}$\pm$\textbf{0.001} {\tiny {\color{blue} (1)}}
&\textbf{0.003}$\pm$\textbf{0.001} {\tiny {\color{blue} (1)}}
&0.246$\pm$0.0016 {\tiny {\color{blue} (7)}}
&{\color{blue} (2)}\\

\midrule 
{\app} \cite{ZLLLG17}		
&0.002$\pm$0.001 {\tiny {\color{blue} (3)}}
&0.012$\pm$0.001 {\tiny {\color{blue} (6)}}
&0.145$\pm$0.001 {\tiny {\color{blue} (2)}}
&{\color{blue} (5)}\\

\midrule 
{\mf} \cite{CHHH11}		
&0.002$\pm$0.001 {\tiny {\color{blue} (3)}}
&0.013$\pm$0.001{\tiny {\color{blue} (7)}}
&0.173$\pm$0.006 {\tiny {\color{blue} (4)}}
&{\color{blue} (7)}\\

\midrule 
{\bpr} \cite{RFGS09}		
&0.002$\pm$0.001 {\tiny {\color{blue} (3)}}
&0.021$\pm$0.001 {\tiny {\color{blue} (8)}}
&0.235$\pm$0.023 {\tiny {\color{blue} (6)}}
&{\color{blue} (8)}\\

\bottomrule
\end{tabular}\label{tab:foursquare_result}
\vspace{-10pt}
\end{table*}
\begin{table*}[t]
\caption{Experimental Results on the Twitter Social Network Dataset.}
\scriptsize
\centering
\setlength{\tabcolsep}{4pt}
\begin{tabular}{c  cccc   }
\toprule
&\multicolumn{4}{c}{Twitter Network} \\
\cmidrule(lr){2-5}
methods &{MSE} &{MAE} &{LRS} &{{ avg. rank}}\\
\midrule 
{{\our}}	
&0.001$\pm$0.001 {\tiny {\color{blue} (1)}}
&\textbf{0.002}$\pm$\textbf{0.001} {\tiny {\color{blue} (1)}}
&\textbf{0.276}$\pm$\textbf{0.0013} {\tiny {\color{blue} (1)}} 
&{\color{blue} (1)}\\

\midrule 
{\deepwalk} \cite{PAS14}	
&0.001$\pm$0.001 {\tiny {\color{blue} (1)}}
&0.003$\pm$0.001 {\tiny {\color{blue} (2)}}
&0.31$\pm$0.028 {\tiny {\color{blue} (4)}}
&{\color{blue} (3)}\\

\midrule 
{\walklets} \cite{PKS16}		
&0.001$\pm$0.001{\tiny {\color{blue} (1)}}
&0.003$\pm$0.001 {\tiny {\color{blue} (2)}}
&0.309$\pm$0.006 {\tiny {\color{blue} (3)}}
&{\color{blue} (2)}\\

\midrule 
{\linemodel} \cite{TQWZYM15}	
&0.001$\pm$0.001{\tiny {\color{blue} (1)}}
&0.003$\pm$0.001 {\tiny {\color{blue} (2)}}
&0.455$\pm$0.008 {\tiny {\color{blue} (8)}}
&{\color{blue} (6)}\\

\midrule 
{\hpe} \cite{CTLY16}		
&0.001$\pm$0.001{\tiny {\color{blue} (1)}}
&0.003$\pm$0.001 {\tiny {\color{blue} (2)}}
&0.351$\pm$0.001 {\tiny {\color{blue} (6)}}
&{\color{blue} (4)}\\

\midrule 
{\app} \cite{ZLLLG17}		
&0.001$\pm$0.0{\tiny {\color{blue} (1)}}
&0.015$\pm$0.0 {\tiny {\color{blue} (7)}}
&0.308$\pm$0.0018 {\tiny {\color{blue} (2)}}
&{\color{blue} (5)}\\

\midrule 
{\mf} \cite{CHHH11}		
&0.001$\pm$0.001{\tiny {\color{blue} (1)}}
&0.013$\pm$0.001 {\tiny {\color{blue} (6)}}
&0.341$\pm$0.0017 {\tiny {\color{blue} (5)}}
&{\color{blue} (7)}\\

\midrule 
{\bpr} \cite{RFGS09}		
&0.002$\pm$0.001 {\tiny {\color{blue} (8)}}
&0.023$\pm$0.001 {\tiny {\color{blue} (8)}}
&0.374$\pm$0.0015 {\tiny {\color{blue} (7)}}
&{\color{blue} (8)}\\

\bottomrule
\end{tabular}\label{tab:twitter_result}
\vspace{-10pt}
\end{table*}

\subsection{Experimental Results}

The results achieved by the comparison methods on these $4$ different graph datasets are provided in Tables~\ref{tab:douban_result}-\ref{tab:twitter_result}. The best results are presented in a bolded font, and the blue numbers in the table denote the relative rankings of the methods regarding certain evaluation metrics. Compared with the baseline methods, \textit{deep loopy neural network} can achieve the best performance than the baseline methods. For instance in Table~\ref{tab:douban_result}, the MSE obtained by \textit{deep loopy neural network} is $0.035$, which about $7.8\%$ lower than the MSE obtained by the $2nd$ best method, i.e., {\mf}; and $30\%$ lower than that of {\app}. Similar results can be observed for the MAE and LRS metrics. At the right hand side of the tables, we illustrate the average ranking positions achieved by the methods based on these $3$ different evaluation metrics. According to the \textit{avg. rank} shown in Tables~\ref{tab:douban_result}, \ref{tab:imdb_result} and \ref{tab:twitter_result}, {\our} can achieve the best performance consistently for all the metrics based on the Douban, IMDB, Foursquare and Twitter datasets. Although in Table~\ref{tab:foursquare_result}, the \textit{deep loopy neural network} model loses to {\hpe} and {\bpr} regarding the MSE metric, but for the other two metrics, it still outperforms all the baseline methods with significant advantages according to the results.

\vspace{-10pt}
\section{Conclusion}\label{sec:conclusion}
\vspace{-10pt}

In this paper, we have introduced the \textit{deep loopy neural network} model, which is a new deep learning model proposed for the graph structured data specifically. Due to the extensive connections among nodes in the input graph data, the constructed \textit{deep loopy neural network} model is very challenging to learn with the traditional error back-propagation algorithm. To resolve such a problem, we propose to extract a set of \textit{k-hop subgraph} and \textit{k-hop rooted spanning tree} from the model architecture, via which the errors can be effectively propagated throughout the model architecture. Extensive experiments have been done on several different categories of graph datasets, and the numerical experiment results demonstrate the effectiveness of both the proposed \textit{deep loopy neural network} model and the introduced learning algorithm.

\bibliographystyle{abbrv}
\bibliography{reference}

\end{document}